\documentclass[11pt,letterpaper]{article}
\usepackage[top=1in, bottom=1in, left=1in, right=1in]{geometry}

\usepackage[utf8]{inputenc} 
\usepackage[T1]{fontenc}   
\usepackage{hyperref}    
\usepackage{url}   
\usepackage{booktabs} 
\usepackage{amsfonts}
\usepackage{nicefrac}
\usepackage{microtype} 
\usepackage{xcolor}
\usepackage{lmodern}

\usepackage{subcaption}
\usepackage[font=small,labelfont=bf]{caption}
\usepackage{multicol}
\usepackage{graphicx}

\RequirePackage{fancyhdr}
\RequirePackage{xcolor}
\RequirePackage{algorithm}
\RequirePackage{algorithmic}
\RequirePackage{natbib}
\RequirePackage{eso-pic}
\RequirePackage{forloop}
\RequirePackage{url}

\usepackage{amsmath, amsthm, amssymb}
\usepackage{bm}
\usepackage{color}
\usepackage{ulem}
\usepackage{mathtools}

\let\emph\textit

\DeclarePairedDelimiterX{\infdivx}[2]{(}{)}{%
    #1\;\delimsize\|\;#2%
}
\newcommand{\infdiv}{\infdivx}

\newcommand{\const}[1]{\mathrm{#1}}
\def\R{{\mathbb{R}}}

\def\1{\textbf{1}}
\def\econst{\const{e}}

\newcommand{\onevct}{\bm{1}}

\newcommand{\zeromtx}{\mathbf{0}}
\newcommand{\Imtx}{\mathcal{I}}

\newtheorem{theorem}{Theorem}
\newtheorem{corollary}{Corollary}
\newtheorem{lemma}{Lemma}
\newtheorem{problem}{Problem}

\theoremstyle{definition}
\newtheorem{definition}{Definition}

\theoremstyle{remark}
\newtheorem{remark}{Remark}

\newcommand{\Prob}[2][]{\mathbb{P}_{#1}\left\{ {#2} \right\}}
\newcommand{\Expect}[2][]{\mathbb{E}_{#1}\left[ #2 \right]}

\newcommand{\kl}[2][]{\mathbb{KL}\infdiv{#1}{#2}}

\newcommand{\abs}[1]{\left\vert {#1} \right\vert}

\newcommand{\diag}[1]{\operatorname*{diag}\left({#1}\right)}
\newcommand{\Diag}[1]{\operatorname*{Diag}\left({#1}\right)}
\newcommand{\tr}[1]{\operatorname*{tr}\left({#1}\right)}

\newcommand{\st}{\operatorname*{subject\; to}}
\newcommand{\maximize}{\operatorname*{maximize}}
\newcommand{\minimize}{\operatorname*{minimize}}

\newcommand{\lmin}{\operatorname{\lambda_{\min}}}
\newcommand{\lsec}{\lambda_{2}}

\newcommand{\inprod}[2][]{\left\langle {#1},{#2} \right\rangle}

\newcommand{\ber}[1]{\texttt{Ber}({#1})}

\newcommand{\yast}{{y^{\ast}}}
\newcommand{\yhat}{{\hat{y}}}

\newcommand{\Vcal}{\mathcal{V}}
\newcommand{\Ecal}{\mathcal{E}}

\newcommand{\Gcal}{\mathcal{G}}

\newcommand{\Scomp}{{\overline{S}}}

\usepackage{bbm}
\newcommand{\onefun}{\mathbbm{1}}
\newcommand{\onepmfun}{{\mathbbm{1}_\pm}}

\makeatletter
\def\algbackskip{\hskip-\ALG@thistlm}
\makeatother

\title{Partial Inference in Structured Prediction}

\author{
  \textbf{Chuyang Ke}\\Department of Computer Science\\Purdue University\\\texttt{cke@purdue.edu}
  \and 
  \textbf{Jean Honorio}\\Department of Computer Science\\Purdue University\\\texttt{jhonorio@purdue.edu}
}

\date{}
\begin{document}
\maketitle

\begin{abstract}
In this paper, we examine the problem of partial inference in the context of structured prediction. Using a generative model approach, we consider the task of maximizing a score function with unary and pairwise potentials in the space of labels on graphs. Employing a two-stage convex optimization algorithm for label recovery, we analyze the conditions under which a majority of the labels can be recovered. We introduce a novel perspective on the Karush-Kuhn-Tucker (KKT) conditions and primal and dual construction, and provide statistical and topological requirements for partial recovery with provable guarantees.
\end{abstract}

\allowdisplaybreaks

\section{Introduction}
In the past decades, various forms of \emph{structured prediction} have been used extensively across many fields, including computer vision, natural language processing, network analysis, computational chemistry, to name a few. 
In these fields, examples of structured prediction problems include foreground / background detection in a digital image \citep{nowozin2011structured}, grammatical part-of-speech tagging in an English sentence \citep{weiss2010structured}, community identification and clustering in social networks \citep{kelley2012defining}, and identifying representative subsets of millions of chemical compounds \citep{downs2002clustering}. 

On a higher level, all of the structured prediction inference problems mentioned above seek to maximize some score function over the space of labels. In other words, a common goal in inference tasks is to recover the label of each entity, such that the prediction matches the observation as much as possible.
Suppose we represent the structured prediction inference problem using an undirected graph $\Gcal = (\Vcal, \Ecal)$, where each node represents an entity, and each edge represents the interaction between two nodes. In the context of Markov random fields (MRFs), inference can be formulated as solving the following optimization problem \citep{bello2019exact}:
\begin{align}
\maximize_{y \in \mathcal{L}^{\abs{\mathcal{V}}}}
\sum_{v\in \Vcal, l\in\mathcal{L}} c_v(l) \cdot \onefun[y_v = l] 
+ \sum_{\substack{(v_1,v_2)\in \Ecal \\ l_1, l_2\in\mathcal{L}}} c_{v_1,v_2}(l_1,l_2) \cdot \onefun[y_{v_1} = l_1, y_{v_2} = l_2] 
\,,
\label{eq:mrf_graph_opt}
\end{align} 
where $\mathcal{L}$ is the space of labels, $c_v(l)$ is the score of assigning label $l$ to node $v$, and $c_{v_1,v_2}(l_1,l_2)$ is the score of assigning labels $l_1$ and $l_2$ to neighboring nodes $v_1$ and $v_2$. These two score terms in \eqref{eq:mrf_graph_opt} are often called unary and pairwise potentials in the MRF and inference literature. 
The optimization formulation above aims to recover the global label structure, by finding a configuration that maximizes the summation of unary and pairwise scores across the graph.

Some prior literature has explored structured prediction and inference problems that involve unary and pairwise potentials, as demonstrated by equation \eqref{eq:mrf_graph_opt}.
For instance, \citet{globerson2015hard} investigated label recovery in two-dimensional grid lattices.
Similarly, \citet{foster2018inference} extended this model to include tree decompositions.
On another note, \citet{bello2019exact} proposed a convex semidefinite programming (SDP) approach to exact inference.
All these works were motivated by a generative model that assumes a ground truth label vector $\yast$ and generates potentially noisy unary and pairwise observations based on label interactions.

In this paper, we follow the two-stage convex optimization approach proposed in \citet{bello2019exact}, but turn our attention on the goal of \emph{partial inference}. Our choice involves two levels of significance:
\begin{itemize}

    \item From an optimization point of view, the primal and dual construction and Karush-Kuhn-Tucker (KKT) analysis have been widely used in generative statistical models to study the guarantee of \emph{exact recovery} in the past decade. 
    Besides structured prediction, examples include sparsity recovery (often referred to as primal-dual witness or PDW) \citep{wainwright2009sharp,ravikumar2011high} and community detection \citep{abbe2015exact, amini2018semidefinite}.
    
    The standard approach of the primal and dual analysis for generative models usually involves three steps. First, it is assumed that there exists some unobserved groundtruth labeling, which generates the observations in a stochastic and noisy fashion (generative assumption). Next, one lists all KKT conditions behind the primal and dual convex optimization problems. In particular, the groundtruth labeling should be a feasible solution to the primal problem. After that, one can analyze the statistical conditions for the groundtruth labeling to be the optimal and unique solution to the optimization problem with high probability. 
    
    While being extremely powerful, the primal and dual approach was only applied to study the problem of exact recovery in the aforementioned literature. 
    In this work, we are interested in extending the primal-dual framework to study the guarantee of \emph{partial recovery}.

    \item On an application level, recovering the majority of the labels can be more relevant and practical in structured prediction inference tasks.
     
    To see this, for example, consider an inference graph model with an isolated node. In other words, following the definition of \eqref{eq:mrf_graph_opt}, assume there exists some node $v$ whose unary potential is zero, and is not connected to any other node in the graph. In this case, we can never recover the true label of this node better than random guessing, and consequently exact inference can never be achieved with a high probability. In contrast, provable guarantees can be achieved if the goal is to recover the majority ($90\%$ for example) of the labels, and our algorithm and analysis can be more robust to outliers. 
    
    Some prior literature considered tasks that are similar to partial inference in specific types of graphs. For instance, \citet{globerson2015hard} studied Hamming error of recovery in two-dimensional grid lattices, and \citet{foster2018inference} studied graphs allowing tree decompositions. It is worth highlighting that our analysis is general: we provides partial inference results for any type of graphs.
    
\end{itemize} 

\textbf{Summary of our contribution.} Our work is mainly theoretical. We propose a novel primal and dual framework to study the problem of partial inference in structured prediction. Our framework analyze the KKT conditions of the convex optimization problem, and derive the sufficient statistical conditions for partial inference of the majority of the true labels. Furthermore, our result subsumes the classic result of exact inference.  
\section{Preliminaries}

In this section, we formally define the structured prediction inference problem and introduce the notations that will be used throughout the paper.

We use lowercase font (e.g., $a,b,u,v$) to denote scalars and vectors, and uppercase font (e.g., $A,B,C$) to denote matrices. 
We denote the set of real numbers by $\R$.

For any natural number $n$, we use $[n]$ to denote the set $\{1,\dots, n\}$.

\subsection{Generative Model}

We follow the customary generative model setting as used in prior structured prediction inference literature \citep{globerson2015hard,bello2019exact}.

We consider an undirected connected graph $\Gcal = (\Vcal, \Ecal)$, where the number of nodes is $\abs{\Vcal} = n$, and $\Ecal$ denotes the set of edges in the graph. 
Every node is in one of the two classes $\{+1, -1\}$, and we use $\yast \in \{+1, -1\}^n$ be the true class label vector. 
We use $Y^\ast := \yast\yast^\top$ to denote the \emph{label structure} of the model, such that $Y_{ij}^\ast = 1$ if and only if $i$ and $j$ share the same label.

The observation of the model is generated in a noisy way. Let $p$ be the edge noise level parameter of the model, and $q$ be the node noise level parameter, with $0 < p < 0.5$, $0< q < 0.5$.
For every edge $(i,j)\in \Ecal$, the model generates a Rademacher random variable $z_{ij}$, such that $z_{ij} = 1$ with probability $1-p$, and $-1$ with probability $p$.
Let $A \in \{-1, 0, +1\}^{n\times n}$ be a noisy observation of the class structure, such that $A_{ij} = z_{ij} Y^\ast_{ij}$.
In other words, $A_{ij} = y_i^\ast y_j^\ast$ with probability $1-p$, and $-y_i^\ast y_j^\ast$ with probability $p$.
For non-edges $(i,j) \notin \Ecal$, we set $A_{ij} = 0$.
Similarly, for every node $i \in \Vcal$, the model generates a Rademacher random variable $z_{i}$, such that $z_{i} = 1$ with probability $1-p$, and $-1$ with probability $p$.
And we use $w \in \{+1, -1\}^{n}$ to denote the noisy observation of the node labels, where $w_i = z_i y_i^\ast$.
Our goal is to recover the true labels from the observation of $A$ and $w$.

We now summarize the generative model.
\begin{definition}[Structured Prediction Inference]
\textbf{Unknown}: True class labeling vector $\yast = (y_1^\ast,\dots, y_n^\ast)$.
\textbf{Observation}: Noisy class structure observation matrix $A \in \{-1,0,+1\}^{n\times n}$. 
Noisy node label observation vector $w \in \{-1,+1\}^n$.
\textbf{Task}: Infer and recover the correct node labeling vector $\yast$ from the observation $A$ and $w$.
\end{definition}

\subsection{Two-stage Exact Inference Algorithm}
For completeness, here we include the two-stage algorithm from \citet{bello2019exact}.
In the first state, one solves the following optimization problem to recover the label vector up to permutation of the two classes:
\begin{align}
\maximize_y \qquad &\frac{1}{2} y^\top A y \nonumber \\
\st \qquad & y \in \{+1, -1\}^n
\,.
\label{eq:combinatorial}
\end{align}
The combinatorial program is NP-hard, and can be relaxed to the following semidefinite program (SDP):
\begin{align}
\maximize_Y \qquad &\inprod[A]{Y} \nonumber \\
\st \qquad & Y \succeq \zeromtx \nonumber \\
& \diag{Y} = \onevct 
\,.
\label{eq:prelim_primal}
\end{align}

In the case of exact inference, solving \eqref{eq:prelim_primal} leads to two feasible solutions: $\yast$ and $-\yast$. In the second stage one determines the correct sign with the help of node observations:
\begin{align}
\maximize_{y\in\{\yast,-\yast\}} \qquad &w^\top y 
\,.
\label{eq:prelim_second_stage}
\end{align}

\subsection{Problem of Partial Inference through SDP}
In this paper we analyze the problem of partial inference through the two-stage SDP approach. The main technical challenge lies in the first stage, that is, recovering the majority of the label vector up to permutation of the two classes. Here we formally define the problem.

\begin{problem}
Under what conditions can the semidefinite program \eqref{eq:prelim_primal} achieve partial recovery of at least $n/2 \leq k \leq n$ nodes, up to permutation of the two classes?
\end{problem}

More specifically, assume we use $S$ to denote the index set of the $k$ nodes whose label we wish to recover. The goal of partial recovery is to achieve $\yhat_S = y_S^\ast$, where $\yhat \yhat^\top = \hat{Y}$, and $\hat{Y}$ is the solution returned by the first stage semidefinite program \eqref{eq:prelim_primal}.

\subsection{Useful Definitions}
we use $S$ to denote the index set of the $k$ nodes whose label we wish to recover. 
We use $\Scomp := [n] \setminus S$ to denote the complement index set of $S$.

We use $\onevct$ to denote the all-one vector, and $\Imtx$ to denote the identity matrix.

For the underlying undirected graph $\Gcal = (\Vcal, \Ecal)$, we use $\Delta_i$ denote the degree of node $i$ in $\Gcal$, and let $\Delta_{\max}$ denote the maximum degree across all nodes. We use $\phi_\Gcal \geq 0$ to denote the Cheeger constant of graph $\Gcal$. Intuitively, $\phi_\Gcal$ measures the connectivity of the graph, and in the case of exact inference, a greater value of $\phi_\Gcal$ leads to easier recovery \citep{bello2019exact}. 

We use $\diag{\cdot}$ to denote the operation of creating a vector from the diagonal entries of a square matrix. For example, $\diag{\Imtx} = \onevct$.
Oppositely, we use $\Diag{\cdot}$ to denote the operation of creating a diagonal matrix from a given vector, and the off-diagonal entries are set to $0$.
We use $\tr{\cdot}$ to denote the trace of a matrix.

For any observation matrix $A$ of size $n$ and label vector $y \in \{+1,-1\}^n$, we define the signed degree vector $v(y)$ as  
$v(y) = \diag{A y y^\top}$.
Additionally, we denote the signed Laplacian matrix $L(y)$ as  
$L(y) = \Diag{v(y)} - A$.

\section{Technical Lemmas}
In this section, we introduce technical lemmas that later will be used in the proofs. 
\begin{lemma}[Mixed Chernoff Bound]
Let $x_1,\dots, x_n$ be independent Bernoulli random variables, such that $x_1,\dots,x_m \sim \ber{r}$, and $x_{m+1},\dots, x_n \sim \ber{1-r}$. Assume $r \geq 0.5$.
We denote the summation as $s := \sum_{i=1}^n x_i$.
Then we have 
\[
\Prob{s \geq \Expect{s} + t} 
\leq 
\exp\left(-\left(2(n-m)+\frac{m}{2r(1-r)}\right)\frac{t^2}{n^2}\right) \,.
\]
\label{PARTIAL:lemma:chernoff}
\end{lemma}

\section{Partial Recovery through Semidefinite Programming}
In this section, we analyze the two-stage inference algorithm for partial recovery. 
During the initial stage, we leverage only the edge information obtained from $A$, which helps to restrict our solution space to two potential options. 
Then, in the subsequent stage, we use the node information acquired from $w$ to accurately determine the labels of the nodes.

\subsection{First Stage: Inference of the Label Structure}

For readers' convenience, here we restate the semidefinite program for structured prediction inference: 
\begin{align}
\maximize_Y \qquad &\inprod[A]{Y} \nonumber \\
\st \qquad & Y \succeq \zeromtx \nonumber \\
& \diag{Y} = \onevct 
\,.
\label{PARTIAL:eq:primal}
\end{align}
Next we introduce the Lagrangian dual of \eqref{PARTIAL:eq:primal}. Let $B, V$ be the Lagrangian dual variables for the two constraints respectively, where $B \succeq \zeromtx$, and $V$ is a diagonal matrix. 
The primal problem has the following Lagrangian function:
\begin{align*}
L(Y,B,V) 
&:= -\inprod[A]{Y} - \inprod[B]{Y} + \inprod[V]{Y - \Imtx} \\
&= \inprod[V-A-B]{Y} - tr(V) 
\,.
\end{align*} 
Setting gradient to zero with respect to $Y$ gives us
\[
\nabla_Y L(Y,B,V) 
= 
V - A -  B
= 
\zeromtx 
\,.   
\]
This leads to the following dual problem 
\begin{align}
\minimize_V \qquad &\tr{V} \nonumber \\
\st \qquad & V - A \succeq \zeromtx \nonumber \\
& V \text{ is diagonal} 
\,,
\label{PARTIAL:eq:dual}
\end{align}
as well as the following KKT conditions between the primal and the dual problems:
\begin{itemize}
    \item Stationarity: $V - A - B = \zeromtx$.
    \item Primal Feasibility: $Y \succeq \zeromtx$, $\diag{Y} = \onevct$.
    \item Dual Feasibility: $B \succeq \zeromtx$, $V \text{ is diagonal}$.
    \item Complementary Slackness: $\inprod[B]{Y} = 0$.
\end{itemize}

To analyze the statistical conditions for partial recovery, here we present the following lemma about sufficient KKT conditions.

\begin{lemma}[Sufficient KKT Conditions for Partial Recovery]
Let $\yhat \in \{+1,-1\}^n$ denote the target partially recovered label vector, such that $\yhat_S = y_S^\ast$,
where $S$ is the index set of the labels that can be recovered successfully up to permutation of the two classes. 
Then, partial recovery of $\yhat$ can be achieved if
\begin{equation}
\lsec(L(\yhat)) > 0 \,.   
\label{PARTIAL:eq:non_unique}
\end{equation}
\end{lemma}

\begin{proof}
To analyze the conditions of partial recovery, we want to fulfill all KKT conditions using $\yhat$. 
To do so, we set $Y = \yhat \yhat^\top$, $V = \Diag{v(\yhat)}$, and $B = V - A$. It follows that $B =  L(\yhat)$.
We now proceed to check all KKT conditions. 
\begin{itemize}
    \item Stationarity: fulfilled by the construction of $B$.
    \item Primal Feasibility: fulfilled by the construction of $Y$.
    \item Dual Feasibility: $V$ fulfilled. See below for $B\succeq \zeromtx$.
    \item Complementary Slackness: fulfilled by noticing that
    \begin{align*}
        \inprod[B]{Y} 
        &= \inprod[L(\yhat)]{\yhat \yhat^\top} \\
        &= \sum_i \left(v_i(\yhat) - A_{ii} - \sum_{j\neq i} A_{ij} \yhat_i \yhat_j \right) \\
        &= \sum_i \left(\left(\sum_j A_{ij} \yhat_i \yhat_j\right) - A_{ii} - \sum_{j\neq i} A_{ij} \yhat_i \yhat_j \right) \\
        &= 0 \,.
    \end{align*}
\end{itemize}
We have two observations. 
First, from the Complementary Slackness condition, we know $\yhat$ is always an eigenvector of $L(\yhat)$, with the corresponding eigenvalue being $0$. 
Second, from Dual Feasibility, the sufficient and necessary condition for optimality is $L(\yhat) \succeq \zeromtx$. 
Combine the two above, the sufficient and necessary condition for optimality can be written as 
\begin{equation}
\lsec(L(\yhat)) \geq 0 \,.    
\label{PARTIAL:eq:optimality_geq}
\end{equation}

However, on top of the KKT conditions, we also want to ensure the uniqueness of the solution on the recoverable part $S$. We do not require uniqueness for $\Scomp$ part (in that case the problem can be reduced to exact recovery). Enforcing uniqueness of $y_S^\ast$ is equivalent to requiring: for every vector $\tilde{y}$, if $\tilde{y}_S$ is not a multiple of $y_S^\ast$, it must fulfill $\inprod[L(\yhat)]{\tilde{y} \tilde{y}^\top} > 0$. Combining with the optimality condition \eqref{PARTIAL:eq:optimality_geq}, we obtain the following sufficient condition for partial recovery
\begin{equation*}
\lsec(L(\yhat)) > 0 \,.   
\end{equation*}
\end{proof}

We now provide our main theorem about partial inference.

\begin{theorem}
Partial recovery of at least $k$ labels can be achieved through solving the semidefinite program \eqref{PARTIAL:eq:primal}, with probability at least $\sum_{i=k}^n \binom{n}{k} \left(1 - \epsilon(\phi_\Gcal,\Delta_{\max},p,k)\right)$, where
\begin{align*} 
\epsilon(\phi_\Gcal,\Delta_{\max},p,k)
&=
n \exp\left(- \frac{(1-2p)^2 \phi_\Gcal^4}{512 p(1-p)\Delta_{\max}^3 + 11(1-2p)(1-p)\Delta_{\max}\phi_\Gcal^2}\right) \\
& + k \exp\left(- \frac{2(1-2p)^2}{\Delta_{\max}} \left(\frac{\phi_\Gcal^2}{16\Delta_{\max}} - (n-k)\right)^2   \right) \nonumber \\
& +  (n-k) \exp\left(- \frac{2(1-2p)^2}{\Delta_{\max}} \left(\frac{\phi_\Gcal^2}{16\Delta_{\max}} - (2k+\Delta_{\max}-n)\right)^2   \right)
\,.    
\end{align*}
\label{PARTIAL:thm_rate}
\end{theorem}

\begin{proof}
On a high level, our proof involves two steps. 
In the first step, we analyze the probability of recovering a specific partially correct vector $\yhat$, such that $\yhat_S = y_S^\ast$, and $\yhat_\Scomp = -y_\Scomp^\ast$. 
In the second step, we apply combinatorics and union bounds to bound the probability of recovering at least $k$ nodes.

We first analyze the conditions under which \eqref{PARTIAL:eq:non_unique} holds for a specific $\yhat$ with $\abs{S} = k$ recoverable labels. 
As stated above, it is sufficient to require that 
$\lsec(L(\yhat)) > 0$ to guarantee partial recovery on $S$.
Note that 
\begin{align}
\lsec(L(\yhat)) 
&= \lsec(\Diag{v(\yhat)} - A) \nonumber\\ 
&= \lsec\left((\Diag{v(\yhat)} - \Expect{\Diag{v(y^\ast)}}) - (A - \Expect{A}) + (\Expect{\Diag{v(y^\ast)} - A})\right) \,.
\label{PARTIAL:eq:three_term_together}
\end{align}
To show \eqref{PARTIAL:eq:three_term_together} being strictly positive, it is sufficient to  prove that 
\begin{align}
\lmin(\Diag{v(\yhat)} - \Expect{\Diag{v(y^\ast)}})
+ \lmin(\Expect{A}- A)
+ \lsec(\Expect{\Diag{v(y^\ast)} - A})
> 0 
\,.
\label{PARTIAL:eq:three_term_splitted}
\end{align}

Regarding the last term in \eqref{PARTIAL:eq:three_term_splitted}:
following \citet[Theorem 2]{bello2019exact}, we obtain that 
\begin{align}
\lsec(\Expect{\Diag{v(y^\ast)} - A}) 
\geq 
(1-2p) \frac{\phi_\Gcal^2}{4\Delta_{\max}}
\,.
\label{PARTIAL:eq:term_expectation}
\end{align}

Regarding the middle term in \eqref{PARTIAL:eq:three_term_splitted}:
using matrix Bernstein's inequality \citep{tropp2015introduction} and set deviation to be half of \eqref{PARTIAL:eq:term_expectation}, we obtain that 
\begin{align}
&\Prob{\lmin(\Expect{A}- A) \leq -(1-2p)\frac{\phi_\Gcal^2}{8\Delta_{\max}}} \nonumber \\
&\leq 
n \exp\left(- \frac{(1-2p)^2 \phi_\Gcal^4}{512 p(1-p)\Delta_{\max}^3 + 11(1-2p)(1-p)\Delta_{\max}\phi_\Gcal^2}\right)
\,.
\label{PARTIAL:eq:term_A_deviation}
\end{align}

Finally we analyze the first term in \eqref{PARTIAL:eq:three_term_splitted}. Since both matrices are diagonal, to find a lower bound for the minimum eigenvalue, it is sufficient to find the minimum of $v_i(\yhat) - \Expect{v_i(y^\ast)}$ among all possible $i\in [n]$ and all possible partial solution $\yhat$.
By definition of $v$, we have 
\begin{align*}
v_i(\yhat) 
&= (A\yhat \yhat^\top)_{ii} \\
&= \yhat_i \sum_{j=1}^n A_{ij} \yhat_j \\
&= \yhat_i y_i^\ast \sum_{j=1}^n z_{ij} \onefun[(i,j)\in \Ecal] y_j^\ast \yhat_j \\
&= \onepmfun[i\in S] \sum_{j=1}^n z_{ij} \onefun[(i,j)\in \Ecal] \onepmfun[j\in S]
\,.
\end{align*} 
Similarly, we have 
\begin{align*}
v_i(\yast) 
&= \sum_{j=1}^n z_{ij} \onefun[(i,j)\in \Ecal]
\,,
\end{align*} 
and 
\begin{align*}
\Expect{v_i(\yast)} 
&= (1-2p) \Delta_i
\,.
\end{align*}
We now proceed to bound 
\begin{align}
v_i(\yhat) - \Expect{v_i(\yast)}
= 
\left(\onepmfun[i\in S] \sum_{j=1}^n z_{ij} \onefun[(i,j)\in \Ecal] \onepmfun[j\in S] \right) - (1-2p) \Delta_i
\,.
\label{eq:v_our_subtraction}
\end{align}
For any fixed $i$, we use $c_{i,S}$ to denote the number of indices $j$ fulfilling $j\in S$ with $(i,j)\in \Ecal$, and we use $c_{i,\Scomp}$ to denote the number of indices $j$ fulfilling $j\notin S$ with $(i,j)\in \Ecal$. 
By counting, we have the lower and upper bounds 
$c_{i,S} \geq \max(0, \Delta_i - (n-\abs{S}))$ and 
$c_{i,\Scomp} \leq \min(\Delta_i, n - \abs{S})$, respectively.
Now we discuss two cases of $i$:
\begin{itemize}
\item If $i\in S$: from \eqref{eq:v_our_subtraction} we obtain 
\begin{align*}
v_i(\yhat) - \Expect{v_i(\yast)} 
&= \left(\sum_{j=1}^n z_{ij} \onefun[(i,j)\in \Ecal] \onepmfun[j\in S] \right) - (1-2p) \Delta_i \\
&= \sum_{(i,j)\in\Ecal} z_{ij} \onepmfun[j\in S]  - (1-2p) \Delta_i  \\
&= \sum_{(i,j)\in\Ecal} \left(z_{ij} - \Expect{z_{ij}}\right) \onepmfun[j\in S] + \sum_{(i,j)\in\Ecal} \Expect{z_{ij}} \onepmfun[j\in S] - (1-2p) \Delta_i  \\
&= \sum_{(i,j)\in\Ecal} \left(z_{ij} - \Expect{z_{ij}}\right) \onepmfun[j\in S] 
+ (1-2p) c_{i,S}
- (1-2p) c_{i,\Scomp}
- (1-2p) \Delta_i  \\
&= \sum_{(i,j)\in\Ecal} \left(z_{ij} - \Expect{z_{ij}}\right) \onepmfun[j\in S]  
-(1-2p) \left(\Delta_i - c_{i,S} + c_{i,\Scomp} \right) \\
&= \sum_{(i,j)\in\Ecal} \left(z_{ij} - \Expect{z_{ij}}\right) \onepmfun[j\in S]  
- 2(1-2p) c_{i,\Scomp} 
\,.
\end{align*} 
Note that the last term is the summation of $\Delta_i$ independent random variables. 
Applying Lemma \ref{PARTIAL:lemma:chernoff}, we obtain that 
\begin{align*}
&\Prob{v_i(\yhat) - \Expect{v_i(\yast)} \geq - (1-2p) \frac{\phi_\Gcal^2}{8\Delta_{\max}}} \\
=& \Prob{\sum_{(i,j)\in\Ecal} \left(z_{ij} - \Expect{z_{ij}}\right) \onepmfun[j\in S]  
\geq 2(1-2p) c_{i,\Scomp} - (1-2p) \frac{\phi_\Gcal^2}{8\Delta_{\max}}} \\
\geq &
1 - \exp\left(-(1-2p)^2 \left(2\Delta_i - 2c_{i,S} + \frac{c_{i,S}}{2p(1-p)}\right)  \left(- c_{i,\Scomp} + \frac{\phi_\Gcal^2}{16 \Delta_{\max}}\right)^2  \frac{1}{\Delta_i^2}\right) \\
\geq &
1 - \exp\left(-2(1-2p)^2 \left(- c_{i,\Scomp} + \frac{\phi_\Gcal^2}{16 \Delta_{\max}}\right)^2  \frac{1}{\Delta_i}\right) \\
\geq &
1 - \exp\left(-2(1-2p)^2 \left(- (n-\abs{S}) + \frac{\phi_\Gcal^2}{16 \Delta_{\max}}\right)^2  \frac{1}{\Delta_i}\right) \\
\geq &
1 - \exp\left(-2(1-2p)^2 \left(- (n-\abs{S}) + \frac{\phi_\Gcal^2}{16 \Delta_{\max}}\right)^2  \frac{1}{\Delta_{\max}}\right) \\
\end{align*}

\item If $i\notin S$: we have 
\begin{align*}
v_i(\yhat) - \Expect{v_i(\yast)} 
&= -(1-2p) \left(\Delta_i + c_{i,S} - c_{i,\Scomp} \right) - \sum_{(i,j)\in\Ecal}  \onepmfun[j\in S] \left(z_{ij} - \Expect{z_{ij}}\right) \,.
\end{align*}
Using the same concentration inequality as above, we obtain    
\begin{align*}
&\Prob{v_i(\yhat) - \Expect{v_i(\yast)} \geq -(1-2p) \frac{\phi_\Gcal^2}{8\Delta_{\max}}} \\
=& \Prob{\sum_{(i,j)\in\Ecal} \left(z_{ij} - \Expect{z_{ij}}\right) \onepmfun[j\in S]  
\geq 2(1-2p) \left(\Delta_i + c_{i,S} - c_{i,\Scomp} \right) - (1-2p) \frac{\phi_\Gcal^2}{8\Delta_{\max}}} \\
\geq &
1 - \exp\left(-(1-2p)^2 \left(2\Delta_i - 2c_{i,S} + \frac{c_{i,S}}{2p(1-p)}\right)  \left(-(\Delta_i + c_{i,S} - c_{i,\Scomp}) + \frac{\phi_\Gcal^2}{16 \Delta_{\max}}\right)^2 \frac{1}{\Delta_i^2}\right) \\
\geq &
1 - \exp\left(-2(1-2p)^2  \left(-(\Delta_i + c_{i,S} - c_{i,\Scomp}) + \frac{\phi_\Gcal^2}{16 \Delta_{\max}}\right)^2 \frac{1}{\Delta_i}\right) \\
\geq &
1 - \exp\left(-2(1-2p)^2  \left(-(2\abs{S} - n +\Delta_{\max}) + \frac{\phi_\Gcal^2}{16 \Delta_{\max}}\right)^2 \frac{1}{\Delta_{\max}}\right) 
\,.
\end{align*}
\end{itemize}
Combining the two parts above, with a union bound we obtain that 
\begin{align}
&\Prob{\lmin(\Diag{v(\yhat)} - \Expect{\Diag{v(y^\ast)}}) \geq -(1-2p) \frac{\phi_\Gcal^2}{8\Delta_{\max}}} \nonumber \\
=& \Prob{\min_{i\in[n]} v_i(\yhat) - \Expect{v_i(\yast)} \geq -(1-2p) \frac{\phi_\Gcal^2}{8\Delta_{\max}}} \nonumber \\
\geq & 1 - \abs{S} \exp\left(- \frac{2(1-2p)^2}{\Delta_{\max}} \left(\frac{\phi_\Gcal^2}{16\Delta_{\max}} - (n-\abs{S})\right)^2   \right) \nonumber \\
&-  (n-\abs{S}) \exp\left(- \frac{2(1-2p)^2}{\Delta_{\max}} \left(\frac{\phi_\Gcal^2}{16\Delta_{\max}} - (2\abs{S}+\Delta_{\max}-n)\right)^2   \right)
\,.
\label{PARTIAL:eq:term_V_deviation}
\end{align}

Combining \eqref{PARTIAL:eq:term_A_deviation} and \eqref{PARTIAL:eq:term_V_deviation}, the probability of recovering a specific $\yhat$ with $\abs{S} = k$ labels is 
\begin{equation}
\Prob{\text{recovering a specfic configuration of }\hat{y}\text{ with $\abs{S} = k$}} 
\geq 
1 - \epsilon(\phi_\Gcal,\Delta_{\max},p,k) 
\,.
\label{eq:rate_singleterm}   
\end{equation}

In the second step, we proceed to bound the probability of recovering at least $k$ nodes. 
Note that there are $\binom{n}{k}$ different configurations to recover $k$ out of $n$ labels, and the events of obtaining each configuration are all disjoint.
Thus, the probability of recovering exactly $k$ labels for any configuration is 
\begin{equation}
\Prob{\text{recovering exactly $k$ labels}} 
\geq 
\binom{n}{k} \left(1 - \epsilon(\phi_\Gcal,\Delta_{\max},p,k)\right)
\,.
\label{eq:rate_k}   
\end{equation}
Additionally, the events of recovering exactly $k, k+1, \dots, n$ labels are all disjoint.
As a result, the probability of recovering at least $k$ labels for any configuration is
\begin{equation}
\Prob{\text{recovering at least $k$ labels}} 
\geq 
\sum_{i=k}^n \binom{n}{k} \left(1 - \epsilon(\phi_\Gcal,\Delta_{\max},p,k)\right)
\,.
\label{eq:rate_geqk}   
\end{equation}

\end{proof}

\begin{remark}
Our partial inference analysis and results subsume the exact inference case.
In particular, if $k = n$, our result in Theorem \ref{PARTIAL:thm_rate} matches the rate of exact inference in \citet[Theorem 2]{bello2019exact}. 
\end{remark}

\subsection{Second Stage: Identification of the Sign}
The first stage returns two potential solutions $\{\hat{y},-\hat{y}\}$ that cannot be differentiated from the observation of $A$ alone. To determine the correct sign of the solution, we use the additional information from $w$. 

\begin{theorem}
Let $\{\hat{y},-\hat{y}\}$ be the two possible solutions output by the first stage. 
The correct vector recovering at least $k$ labels can be inferred from program 
\begin{align}
\maximize_{y\in\{\yast,-\yast\}} \qquad &w^\top y 
\label{eq:second_stage}
\end{align}
with probability at least $1-\epsilon_2(n,q)$, where 
\begin{equation}
\epsilon_2(n,q) = \econst^{-(1-2q^2)n/2} \,.    
\end{equation}
\label{thm:stage2rate}
\end{theorem}
\section{Example Graphs for Partial Inference}
In this section, we present some example graphs that can be recovered by our approach. Theorem \ref{PARTIAL:thm_rate} guarantees that partial inference can be achieved with high probability. In particular, our results of partial inference on the following classes of graphs subsume the results of exact inference from \citet{bello2019exact}.

\begin{corollary}[Complete Graphs]
Let $\Gcal$ be a complete graph of $n$ nodes. Then, partial inference of any number of nodes is achievable with high probability. In other words, exact inference is achievable with high probability.   
\label{cor:complete} 
\end{corollary}

\begin{definition}[$d$-regular Expander \citep{bello2019exact}]
$\Gcal$ is a $d$-regular expander graph with constant $c$, if for every set $S\subset \Vcal$ with $\abs{S} \leq n/2$, the number of edges connecting $S$ and $\Vcal \setminus S$ is greater than or equal to $c\cdot d \cdot \abs{S}$.
\label{def:expander}
\end{definition}

\begin{corollary}[Expander Graphs]
Let $\Gcal$ be a $d$-regular expander graph with constant $c$. Then, partial inference of $k$ nodes is achievable with high probability if $d\in\Omega(\log n)$, $(c^2 d / 16 - (n-k))^2\in \Omega(d \log k)$, and $(c^2 d / 16 - (2k+d-n))^2\in \Omega(d \log (n-k))$.
\label{cor:expander}
\end{corollary}
\section{Simulation Results}
We test the proposed approach on synthetic graphs, and check how many labels can be recovered experimentally. 
We fix the number of nodes to be $n = 500$, and test with different values of $p$.
For each setting we run $1000$ iterations. 
See Figure \ref{fig:simulation} for the results.
Matching our theoretical findings, our results suggest that if the noise level $p$ is smaller, or if the target number of labels to be recovered is smaller, the probability of achieving partial recovery is greater.

\begin{figure}[ht!]
\centering
\includegraphics[width=0.6\linewidth]{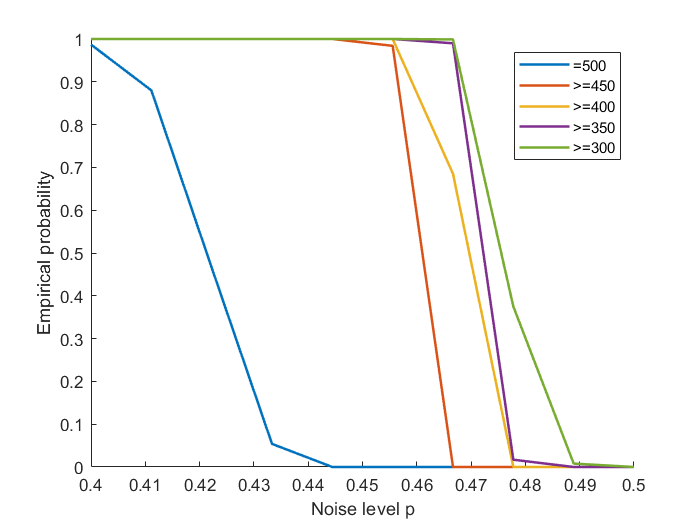}
\caption{Simulations with different noise levels $p$. Our results suggest that if the noise level $p$ is smaller, or if the target number of labels to be recovered is smaller, the probability of achieving partial recovery is greater.}
\label{fig:simulation}
\end{figure}

\bibliography{0_main.bib}

\begin{thebibliography}{12}
\providecommand{\natexlab}[1]{#1}
\providecommand{\url}[1]{\texttt{#1}}
\expandafter\ifx\csname urlstyle\endcsname\relax
  \providecommand{\doi}[1]{doi: #1}\else
  \providecommand{\doi}{doi: \begingroup \urlstyle{rm}\Url}\fi

\bibitem[Abbe et~al.(2015)Abbe, Bandeira, and Hall]{abbe2015exact}
Emmanuel Abbe, Afonso~S Bandeira, and Georgina Hall.
\newblock Exact recovery in the stochastic block model.
\newblock \emph{IEEE Transactions on information theory}, 62\penalty0
  (1):\penalty0 471--487, 2015.

\bibitem[Amini and Levina(2018)]{amini2018semidefinite}
Arash~A Amini and Elizaveta Levina.
\newblock On semidefinite relaxations for the block model.
\newblock \emph{The Annals of Statistics}, 46\penalty0 (1):\penalty0 149--179,
  2018.

\bibitem[Bello and Honorio(2019)]{bello2019exact}
Kevin Bello and Jean Honorio.
\newblock Exact inference in structured prediction.
\newblock \emph{Advances in Neural Information Processing Systems}, 32, 2019.

\bibitem[Downs and Barnard(2002)]{downs2002clustering}
Geoff~M Downs and John~M Barnard.
\newblock Clustering methods and their uses in computational chemistry.
\newblock \emph{Reviews in computational chemistry}, 18:\penalty0 1--40, 2002.

\bibitem[Foster et~al.(2018)Foster, Sridharan, and
  Reichman]{foster2018inference}
Dylan Foster, Karthik Sridharan, and Daniel Reichman.
\newblock Inference in sparse graphs with pairwise measurements and side
  information.
\newblock In \emph{International Conference on Artificial Intelligence and
  Statistics}, pages 1810--1818. PMLR, 2018.

\bibitem[Globerson et~al.(2015)Globerson, Roughgarden, Sontag, and
  Yildirim]{globerson2015hard}
Amir Globerson, Tim Roughgarden, David Sontag, and Cafer Yildirim.
\newblock How hard is inference for structured prediction?
\newblock In \emph{International Conference on Machine Learning}, pages
  2181--2190. PMLR, 2015.

\bibitem[Kelley et~al.(2012)Kelley, Goldberg, Magdon-Ismail, Mertsalov, and
  Wallace]{kelley2012defining}
Stephen Kelley, Mark Goldberg, Malik Magdon-Ismail, Konstantin Mertsalov, and
  Al~Wallace.
\newblock Defining and discovering communities in social networks.
\newblock In \emph{Handbook of Optimization in Complex Networks}, pages
  139--168. Springer, 2012.

\bibitem[Nowozin et~al.(2011)Nowozin, Lampert, et~al.]{nowozin2011structured}
Sebastian Nowozin, Christoph~H Lampert, et~al.
\newblock Structured learning and prediction in computer vision.
\newblock \emph{Foundations and Trends{\textregistered} in Computer Graphics
  and Vision}, 6\penalty0 (3--4):\penalty0 185--365, 2011.

\bibitem[Ravikumar et~al.(2011)Ravikumar, Wainwright, Raskutti, and
  Yu]{ravikumar2011high}
Pradeep Ravikumar, Martin~J Wainwright, Garvesh Raskutti, and Bin Yu.
\newblock High-dimensional covariance estimation by minimizing l1-penalized
  log-determinant divergence.
\newblock \emph{Electronic Journal of Statistics}, 5:\penalty0 935--980, 2011.

\bibitem[Tropp et~al.(2015)]{tropp2015introduction}
Joel~A Tropp et~al.
\newblock An introduction to matrix concentration inequalities.
\newblock \emph{Foundations and Trends{\textregistered} in Machine Learning},
  8\penalty0 (1-2):\penalty0 1--230, 2015.

\bibitem[Wainwright(2009)]{wainwright2009sharp}
Martin~J Wainwright.
\newblock Sharp thresholds for high-dimensional and noisy sparsity recovery
  using l1-constrained quadratic programming (lasso).
\newblock \emph{IEEE transactions on information theory}, 55\penalty0
  (5):\penalty0 2183--2202, 2009.

\bibitem[Weiss and Taskar(2010)]{weiss2010structured}
David Weiss and Benjamin Taskar.
\newblock Structured prediction cascades.
\newblock In \emph{Proceedings of the Thirteenth International Conference on
  Artificial Intelligence and Statistics}, pages 916--923. JMLR Workshop and
  Conference Proceedings, 2010.

\end{thebibliography}
\bibliographystyle{plainnat}

\newpage
\appendix
\section{Proofs}

\begin{proof}[Proof of Lemma \ref{PARTIAL:lemma:chernoff}]
For any $t \geq 0$ and $\lambda > 0$, note that 
\begin{align*}
\Prob{s \geq \Expect{s} + nt}
&= \Prob{\econst^{\lambda s} \geq \econst^{\lambda(\Expect{s} + nt)}} \\
&\leq \econst^{-\lambda(\Expect{s} + nt)} \Expect{\econst^{\lambda s}} 
\,,
\end{align*}
using the Markov inequality.
For Bernoulli random variables, note that we have the moment generating function 
\[
\Expect{\econst^{\lambda s}} 
= 
(r\econst^\lambda + 1 - r)^m ((1-r)\econst^\lambda + r)^{n-m}    
\,.
\]
Also note that
\[
\econst^{\lambda(\Expect{s} + nt)}
= 
(\econst^{\lambda (r+t)})^m (\econst^{\lambda (1-r+t)})^{n-m}    
\,.
\]
This leads to 
\begin{align}
\Prob{s \geq \Expect{s} + nt}
&\leq \left(\frac{r\econst^\lambda + 1 - r}{\econst^{\lambda (r+t)}}\right)^m \left(\frac{(1-r)\econst^\lambda + r}{\econst^{\lambda (1-r+t)}}\right)^{n-m}
\,.
\label{PARTIAL:eq:lemma_twoparts}
\end{align}
Here we minimize the two parts independently with respect to $\lambda$. 
For the first part in \eqref{PARTIAL:eq:lemma_twoparts}, the derivative is 
\begin{align*}
\frac{\partial}{\partial \lambda} \left(\frac{r\econst^\lambda + 1 - r}{\econst^{\lambda (r+t)}}\right)
&= \econst^{-\lambda(r+t)} \left(-t + (\econst^\lambda - 1) r (1-r-t)\right) 
\,.
\end{align*}
Setting the derivative to $0$ leads to the minimizer $\lambda^\ast = \log \frac{(1-r)(r+t)}{r(1-r-t)}$, which gives us the bound
\[
\left(\frac{r\econst^\lambda + 1 - r}{\econst^{\lambda (r+t)}}\right)
\leq 
\left(\frac{r}{r+t}\right)^{r+t} \left(\frac{1-r}{1-r-t}\right)^{1-r-t} 
\,.
\]
Similarly, for the second part in \eqref{PARTIAL:eq:lemma_twoparts}, we obtain
\[
\left(\frac{(1-r)\econst^\lambda + r}{\econst^{\lambda (1-r+t)}}\right)
\leq 
\left(\frac{1-r}{1-r+t}\right)^{1-r+t} \left(\frac{r}{r-t}\right)^{r-t} 
\,.
\]
Back to the concentration inequality, we now have
\begin{align*}
\Prob{s \geq \Expect{s} + nt}
&\leq \left(\frac{r\econst^\lambda + 1 - r}{\econst^{\lambda (r+t)}}\right)^m \left(\frac{(1-r)\econst^\lambda + r}{\econst^{\lambda (1-r+t)}}\right)^{n-m} \\
&\leq 
\left(\left(\frac{r}{r+t}\right)^{r+t} \left(\frac{1-r}{1-r-t}\right)^{1-r-t}\right)^m
\left(\left(\frac{1-r}{1-r+t}\right)^{1-r+t} \left(\frac{r}{r-t}\right)^{r-t} \right)^{n-m} \\
&= \exp\left(-m \kl[r+t]{r} - (n-m) \kl[r-t]{r}\right) \\
&\leq \exp\left(-\frac{mt^2}{2r(1-r)} - 2(n-m)t^2\right)
\,.
\end{align*}
The last inequality follows from the fact that 
$\kl[r+t]{r} \geq \frac{t^2}{2r(1-r)}$, and $\kl[r-t]{r} \geq 2t^2$ if $r\geq \frac{1}{2}$ and $t \geq 0$.
Reparametrization of $nt$ as $t$ leads to the result. 
\end{proof}

\begin{proof}[Proof of Theorem \ref{thm:stage2rate}]
This follows directly from Hoeffding's inequality
\begin{align*}
\Prob{w^\top y^\ast \leq -w^\top y^\ast}
&= \Prob{w^\top y^\ast \leq 0} \\
&\leq \econst^{-(1-2q^2)n/2} 
\,.
\end{align*}
\end{proof}

\begin{proof}[Proof of Corollary \ref{cor:complete}]
By \citet{bello2019exact}, we have $\epsilon(\phi_{\Gcal}, \Delta_{\max}, p, n) \to 0$ as $n\to \infty$.
As a result, for any $k$ we have 
$\sum_{i=k}^n \binom{n}{k} \left(1 - \epsilon(\phi_\Gcal,\Delta_{\max},p,k)\right) \to 1$.
\end{proof}

\begin{proof}[Proof of Corollary \ref{cor:expander}]
Using Definition \ref{def:expander}, for a $d$-regular expander graph with constant $c$, we have $\phi_{\Gcal} \geq cd$, and $\phi_{\Gcal}^2 / \Delta_{\max} \in \Omega(d)$. 

If $d\in \Omega(\log n)$, we have $n \exp\left(- \frac{(1-2p)^2 \phi_\Gcal^4}{512 p(1-p)\Delta_{\max}^3 + 11(1-2p)(1-p)\Delta_{\max}\phi_\Gcal^2}\right) \to 0$ as $n \to \infty$.
Similarly, if $(c^2 d / 16 - (n-k))^2\in \Omega(d \log k)$, we have $k \exp\left(- \frac{2(1-2p)^2}{\Delta_{\max}} \left(\frac{\phi_\Gcal^2}{16\Delta_{\max}} - (n-k)\right)^2   \right) \to 0$.
And if $(c^2 d / 16 - (2k+d-n))^2\in \Omega(d \log (n-k)) \to 0$, we have $(n-k) \exp\left(- \frac{2(1-2p)^2}{\Delta_{\max}} \left(\frac{\phi_\Gcal^2}{16\Delta_{\max}} - (2k+\Delta_{\max}-n)\right)^2   \right) \to 0$.

Combining the three terms leads to the summation of errors tending to $0$ as $n\to \infty$. 
\end{proof}

\end{document}